\documentclass[11pt]{article}

\usepackage{acl}

\usepackage{times}
\usepackage{latexsym}
\usepackage[T1]{fontenc}
\usepackage[utf8]{inputenc}
\usepackage{microtype}
\usepackage{inconsolata}
\usepackage{graphicx}
\usepackage{booktabs}
\usepackage{multirow}
\usepackage{amsmath}
\usepackage{amssymb}
\usepackage{amsthm}

\newtheorem{proposition}{Proposition}

\newtheorem{corollary}{Corollary}

\title{Amplifying, Not Learning: \\
  The Price of Out-of-Distribution Generalization in AI-Text Detection}

\author{Alexander Smirnov \\
    University College London \\
    \texttt{alexander.smirnov.25@ucl.ac.uk}}

\newcommand{\dclass}{\mathbf{d}_{\mathrm{class}}}
\newcommand{\dtyp}{\mathbf{d}_{\mathrm{typ}}}
\newcommand{\cls}{\mathbf{h}}
\newcommand{\wh}{\mathbf{w}_h}
\newcommand{\paff}{P(\mathrm{AI})}

\begin{document}
\maketitle

% ===========================================================================
% ABSTRACT  (prose-kit Part C, route-shift-safe; finalize at L3)
% Self-check: no "frozen beats FT" (says "matches") | no "perplexity meter"
% (says "typicality axis") | no "irreducible" | route-shift safe
% ("relocates ... cannot erase") | scoped no-go ("removes ... while
% preserving") | "amplifies/rescales" load-bearing.
% ===========================================================================
\begin{abstract}
AI-text detectors gate decisions in education, hiring, and publishing, yet
they systematically flag the most fluent, formal human writing as
machine-generated: they rate the median formal-native human essay as $99.5\%$
likely AI while clearing genuine high-temperature AI at $10.5\%$. The pattern
is shared by deployed detectors (chatgpt-detector-roberta flags $56\%$ of
formal essays at a $1\%$ false-alarm rate). We show this is not a calibration bug but the signature of one
mechanism: a fine-tuned detector does not learn an AI-versus-human boundary,
it \emph{amplifies} an inherited \emph{typicality} axis (predictability under
a language model) that pre-exists fine-tuning, rescaling it rather than
constructing one.
Decomposing the detector into this inherited reading and a fine-tuned
residual, we find the inherited part carries the bulk of cross-generator
transfer \emph{and} produces the over-flagging of formal humans, while the
residual is generator-specific and does not transfer; indeed a
frozen-representation probe fit on ${\sim}25$ labels per class matches the fully
fine-tuned detector on unseen generators (cross-generator AUROC $0.893$ vs $0.831$). This yields a
no-go: because the axis that transfers is the axis that over-flags, no
training objective, threshold, concept-erasure, or ensemble we test removes
the harm while preserving cross-generator detection, across three architectures,
fine-tuned decoders, and zero-shot perplexity detectors. We give a
closed-form, training-free operator that relocates and diagnoses the bias
(reviving a dead deployed detector, lifting its true-positive rate from $0$ to
$0.904$ at a $1\%$ false-alarm rate) but, consistent with the no-go, is
AUROC-neutral on cross-generator detection: it moves the
bias, it cannot erase it. The mechanism is not specific to English prose, the
over-flag and the inherited-residual decomposition replicate in Chinese and in
code. Detector unfairness is a predictable, structural
property of the detection paradigm, the price of out-of-distribution
generalization.
\end{abstract}

% ===========================================================================
% THROUGH-LINE (restate ~verbatim at each section seam):
%   A fine-tuned detector never learns an "AI" concept -- it READS an inherited
%   typicality signal (predictability under an LM); that signal is at once the
%   only thing that generalizes to unseen generators AND the thing that
%   over-flags fluent/formal humans, so the fairness harm and the detector's
%   robustness are ONE quantity -- the price of OOD generalization -- relocatable
%   but not surgically removable. (One axis, two faces.)
% Three contributions: (1) finding = read-not-learned; (2) tool = closed-form
% -eps operator; (3) scope = paradigm-wide + the no-go everywhere.
% ===========================================================================

\section{Introduction}
\label{sec:intro}

\begin{figure}[t]
  \centering
  \includegraphics[width=\columnwidth]{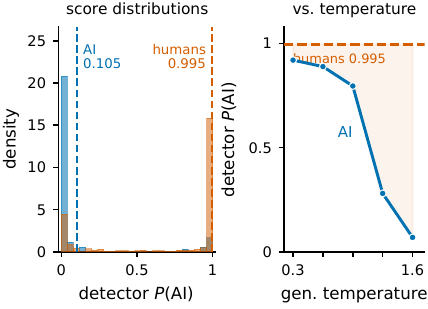}
  \caption{The typicality inversion. \textbf{Left}: a fine-tuned ELECTRA
    detector's score distributions on formal-native human essays (Ghostbuster,
    $n=600$) and genuine high-temperature AI ($n=90$), the human mass concentrates
    near $\paff=0.995$ and the AI mass near $0.105$ (length-controlled).
    \textbf{Right}: the detector's mean score on AI output falls from $0.92$ at low
    generation temperature to $0.07$ at high temperature, staying below the
    formal-human median ($0.995$) throughout. Careful humans are flagged more than
    genuine AI.}
  \label{fig:inversion}
\end{figure}

Automated detectors of machine-generated text now gate consequential
decisions in education, hiring, and publishing. Yet run a state-of-the-art
detector on a careful, formal essay by a native English writer and on a
paragraph drawn from a language model at high temperature, and it more
often calls the \emph{human} essay machine-generated
(Figure~\ref{fig:inversion}): on length-matched samples the median formal
human essay is rated far more likely AI than genuine high-temperature model
output. The detector is not malfunctioning, it is doing
exactly what it was trained to do.
% S1.284 (Ghost median 0.995 vs high-temp AI 0.105; NOT FCE, NOT the mean)

This bites in deployment, on named systems. At its default threshold an
ELECTRA detector \citep{clark2020electra} flags $33.5\%$ of professional
journalism (NYT; \citealp{roy2025comprehensive}) and $70\%$ of
formal-native student essays (Ghostbuster; \citealp{verma2024ghostbuster},
hereafter Ghost). chatgpt-detector-roberta flags $56\%$ of formal essays at
a $1\%$ false-alarm rate.
% S1.267 (NYT 0.335 / Ghost 0.703 @default; hook S1.94/S1.256) / S1.280 (56% @1%)
The harm is sharpest for public detectors (in-house can self-recalibrate),
and at a $5\%$ false-alarm rate the strongest over-flag formal essays at
$71$--$98\%$
(Appendix~\ref{app:fairness}).
% S1.293 (@5% 71-98% = the four strongest only)

The root cause is one sentence: contemporary AI is \emph{more typical} than
professional humans. At matched length, $74$--$98\%$ of 2025-commercial AI
output is lower-perplexity than the median professional human, and
$87$--$98\%$ against professional journalism specifically, which rules out an
informal-register confound, so a detector keyed to typicality \emph{must}
place the most polished humans on the AI side.
% S1.285 (87-98% vs NYT-news = the ELI5/register-confound kill; re-byte-check at L3)

These failures are one mechanism seen from two sides. A fine-tuned detector
does not learn an AI-versus-human boundary; it \emph{amplifies} an inherited
\emph{typicality} axis (predictability under a language model) that pre-exists
fine-tuning, rescaling it rather than constructing one. That single axis is at once the
only signal that transfers to generators the detector never saw and the signal
that over-flags fluent, formal humans. A detector's robustness and its
unfairness are therefore one quantity, the price of out-of-distribution
generalization. We
contribute (i)~a \emph{finding}, that the discriminative axis is read, not
learned; (ii)~a \emph{tool}, a closed-form, training-free operator that revives
dead deployed detectors, diagnoses bias zero-shot on third-party detectors,
and relocates over-flag, AUROC-neutral by construction; and (iii)~a \emph{scope} result, the axis is
paradigm-wide (encoders, decoders to 8B, and zero-shot perplexity detectors),
making the harm a predictable, structural fairness property.

\section{Related Work}
\label{sec:related}
Machine-text detectors are known to misfire on human writing, and unevenly:
audits repeatedly find elevated false-positive rates on non-native English
authors and other populations, and \citet{kuznetsov2024} catalogue detector
robustness and bias among the field's open problems. Prior audits establish
\emph{that} the bias exists; what is missing, and what we supply, is the
mechanism for \emph{why} it arises and whether any detector design removes it.

Two detector families dominate. Supervised detectors fine-tune an encoder on
human/AI labels; zero-shot detectors score text from its statistics under a
reference language model, GLTR \citep{gehrmann2019gltr}, DetectGPT
\citep{mitchell2023detectgpt}, Binoculars \citep{hans2024binoculars}, and
Fast-DetectGPT \citep{bao2024fastdetectgpt} read perplexity or likelihood
curvature with \emph{no trained classifier head}. That a purely
likelihood-based score detects machine text hints that ``AI-ness'' and
typicality are entangled; we ask whether supervised detectors do more than
read the same axis. Largely they do not: a frozen probe matches or beats
zero-shot Binoculars on $18$ of $20$ settings.
% S1.171 (the per-detector head-to-head 0.966 vs 0.159 lives in S4.4, de-duped)

When a representation carries an unwanted direction, the standard remedy is
concept erasure during training: INLP \citep{ravfogel2020inlp}, RLACE
\citep{ravfogel2022rlace}, and LEACE \citep{belrose2023leace}. These pose our
second question: if a bias rides one direction, can it be erased during
training? We contrast this training-time
toolkit with an inference-time, per-text operator (\S\ref{sec:tool}); both
face the no-go. Our impossibility result, in turn, has the shape of a
recognized one: we position it as a detection-specific instance, in the spirit
of the fairness-impossibility theorems \citep{chouldechova2017,kleinberg2017},
with the unusual feature that what plays the role of the protected attribute
\emph{is} the decision statistic. We claim kinship, not their theorem; the
setting differs.

\section{Methods}
\label{sec:methods}
We define the objects the argument needs, introducing notation as it is
used. A \emph{detector} maps a passage $x$ to a score $\paff(x)\in[0,1]$;
unless noted it is an encoder fine-tuned on human/AI labels, with a
final-layer \texttt{[CLS]} representation $\cls(x)$ and a linear decision
head.

\paragraph{The typicality axis, and what ``amplifying'' means.}
On a \emph{frozen}, never-fine-tuned encoder we form the class axis
\begin{equation}
\dclass \;\propto\; \mathrm{centroid}(\mathrm{AI}) \;-\; \mathrm{centroid}(\mathrm{HC3}),
\label{eq:dclass}
\end{equation}
the unit-normalized difference of the mean \texttt{[CLS]} vectors of an AI
population and the human anchor (HC3). We characterize $\dclass$ not as an
``AI-content'' direction but as a \emph{typicality} direction: it orders
text by how predictable it is under a language model, so that fluent,
low-perplexity prose, human or machine, projects toward the AI pole. The
identification is not new in spirit, as zero-shot scores that read
predictability alone separate machine from human text with no trained head
\citep{gehrmann2019gltr,hans2024binoculars,bao2024fastdetectgpt}; what is
new is the causal test on the supervised detector's own representation
(\S\ref{sec:causal}). Typicality here is \emph{internal}, a property of
that representation, for which an external reference-model perplexity is
only a lossy proxy; we therefore never read the detector as a perplexity
meter. Fine-tuning, on this account, \emph{rescales and re-thresholds} the
inherited axis rather than constructing or rotating a new one: it
\emph{amplifies}. Rescaling is not nothing, fine-tuning also adds an
off-axis residual, which we isolate next.

\begin{figure}[t]
  \centering
  \includegraphics[width=\columnwidth]{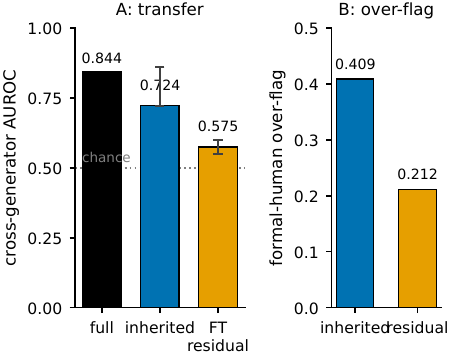}
  \caption{Inherited-vs-residual decomposition. The detector splits into an
    inherited typicality reading and a fine-tuned residual; only the inherited
    part transfers across generators (Panel A: inherited ${\sim}0.72$--$0.86$ vs
    residual ${\sim}0.55$--$0.60$ vs full $0.844$), and it is the part that
    over-flags formal humans (Panel B). One axis carries both.}
  \label{fig:decomp}
\end{figure}

\paragraph{The inherited--residual decomposition.}
We organize the paper around a single construct (Figure~\ref{fig:decomp}):
a fine-tuned detector decomposes into an \emph{inherited} typicality
reading, the component aligned with the pretrained axis, and a fine-tuned
\emph{residual} orthogonal to it. The two behave oppositely. The inherited
component both transfers to generators the detector never saw and over-flags
formal humans; the residual is generator-specific, debiasable, but it does
not survive a change of generator. This one picture unifies the paper's
three claims: read-not-learned (\S\ref{sec:read}), the price, and the no-go
(\S\ref{sec:nogo}).

\paragraph{A closed-form operator and its per-text predictor.}
Our methodological contribution is a training-free knob on the axis paired
with a closed-form forecast of its effect. For the typicality axis $\dtyp$
(estimated within the human anchor as the direction of reference-model
predictability; it aligns with $\dclass$, $\cos\approx0.79$,
Appendix~\ref{app:predictor}) and signed strength $\varepsilon$, writing
$s(x)=\langle \cls(x), \dtyp \rangle$ for the projection onto the axis, we ablate
\begin{equation}
\cls(x) \;\mapsto\; \cls(x) - \varepsilon\, s(x)\, \dtyp,
\label{eq:operator}
\end{equation}
and the per-text logit moves by
\begin{equation}
\Delta\mathrm{logit}(x) \;=\; -\,\varepsilon\, s(x)\, \langle \wh(x), \dtyp \rangle,
\label{eq:predictor}
\end{equation}
where $\wh(x)=\partial\,\mathrm{logit}/\partial\cls$ is the per-text head
Jacobian, computed in one backward pass. This is \emph{exact} for
linear and decoder score heads (where $\wh$ is constant) and first-order
(Taylor) accurate for $\varepsilon\le 0.7$ on fine-tuned encoders (per-text
$R^2=0.993$--$1.000$ on the bias axis).% S1.96.13

\paragraph{Operating points.}
Because the over-flag's \emph{size} depends entirely on where a detector
thresholds, every claim is operating-point-aware. We use four: the default
decision threshold (logit-difference $\ge 0$); a \emph{matched-TPR} point at
which within-domain AI recall is fixed to $0.90$, which removes calibration
as a confound; an \emph{OOD-matched-recall} point fixing recall on held-out
novel-generator AI to $0.90$ (the deployment point the no-go turns on); and
anchored false-alarm points at $1\%$ and $5\%$ FPR. The
same detector can be catastrophic at one point and near-fair at another, so
we never conflate operating points, nor compare absolute magnitudes across
training recipes.

\paragraph{Data and the held-out scope.}
HC3 \citep{guo2023hc3} is the human/AI anchor. Evaluation human populations
span formal journalism (NYT), formal-native student essays
(Ghostbuster), scientific and mixed-domain
text (SemEval-2024 Task~8; \citealp{semeval2024}), and a held-out
cross-generator set (EvoBench; \citealp{evobench2024}); non-native English
(FCE; \citealp{yannakoudakis2011fce}) and dialect
corpora enter in \S\ref{sec:fairness}. A reference LM (GPT-2 / LLaMA-3)
supplies external NLL. Backbones span encoders (ELECTRA, RoBERTa
\citep{liu2019roberta}, DeBERTa \citep{he2021deberta}), with ELECTRA the primary
exposition backbone, and decoders from 124M to 8B parameters. One scope
guard is load-bearing: held-out cross-generator evaluation is EvoBench and
SemEval; \textbf{NYT is in the training distribution}, an in-distribution
positive control rather than transfer, and RAID \citep{raid2024} is in
the training pool, so we cite it only for named-benchmark and adversarial
axes.

\section{Results}
\label{sec:results}
A fine-tuned detector reads an inherited typicality axis that is at once its
only signal for generators it never saw and the source of its over-flagging
of formal humans; the results below establish each half, then show the two
are one quantity.

\subsection{Read, not learned}
\label{sec:read}
The cross-generator signal is inherited, not created by fine-tuning.
A linear probe on the frozen, never-fine-tuned representation already
recovers $76\%$ of the fine-tuned detector's above-chance cross-generator
discrimination ($0.752$ of $0.831$; chance is $0.5$), and fine-tuning the representation barely
raises its within-population separability ($+0.006$).
% S1.294 (clean transfer: frozen 0.752 < FT 0.831; 76% = above-chance ratio; +0.006 = within-pop capacity 0.971->0.977)
This is a shortcut-learning inversion \citep{geirhos2020shortcut}: the
\emph{simple} feature, typicality, is what generalizes, while the
complexity fine-tuning adds is what fails out of distribution. The claim is
not that a frozen probe \emph{beats} fine-tuning: in a clean, same-data
comparison the fully fine-tuned head still leads, on four of five
architectures. What generalizes is read, not learned.

What fine-tuning adds does not generalize. The off-axis residual it
introduces transfers far worse than the inherited reading
($\approx 0.55$--$0.60$ versus $0.72$--$0.86$ cross-generator;
Figure~\ref{fig:decomp}): fine-tuning's own contribution is in-distribution
polish that dies on a new generator.
% S1.271 ; breadth (4 archs x 11 held-out gens) -> appendix
Appendix~\ref{app:nogo} reports this across architectures and held-out
generators.

The inheritance makes detection cheap, in both labels and compute. A
frozen-representation probe fit on about $25$ labeled examples per class,
with no gradient training, matches the fully fine-tuned detector on held-out
generators ($0.893$ versus $0.831$; $50$ labels reach $0.929$), $25$ per class
against full fine-tuning on ${\sim}30{,}000$ samples. Reading typicality directly from a single
110M-parameter encoder, it also matches the zero-shot detectors Binoculars
and Fast-DetectGPT on these generators, using over $100\times$ fewer
parameters than Binoculars, which scores every text with two 7B models.
% S1.294 ; cost: ELECTRA-base 110M vs Binoculars 2x Falcon-7B=14B (S1.171, 27.8GB peak); probe matches Bino/FDG on OOD (few-shot 0.893 > E1 Bino 0.681/FDG 0.705; zero-shot transfer 0.752 also > both)

The carrier of transfer is internal typicality, and only out of
distribution. Matching internal typicality between AI and human text
collapses cross-generator detection from $0.85$ to $0.58$, near chance,
whereas matching an external reference-LM perplexity barely moves it
($0.83$); in-distribution detection on NYT is untouched at $0.99$.
% S1.286 / S1.293
Because the in-distribution detector plainly learns more than typicality,
the unfairness is specifically a property of \emph{out-of-distribution}
generalization, which licenses calling it the price of OOD generalization.

The axis pre-exists fine-tuning paradigm-wide, and on it the decision
is nearly one-dimensional. Across nine decoder language models from 124M to
8B, the raw projection of the un-fine-tuned model separates each model's own AI
output from human text at AUROC $0.93$--$0.99$, matched-length confirmed, none inverted
(Appendix~\ref{app:scope}).
% S1.227
The in-distribution detection rank is $k^{*}{=}1$, rising to $3$--$10$ only
out of distribution (Appendix~\ref{app:nogo}).
% S1.255 / S1.153
This one-dimensionality is the premise the no-go (\S\ref{sec:nogo}) rests on.

\subsection{The axis is typicality, causally}
\label{sec:causal}
Manipulating typicality, not authorship, moves the detector. Raising a
generator's sampling temperature makes its output less typical, and it slips
past the detector: $86$ of $100$ previously caught samples flip to ``human'' (the
length-controlled correlation between reference-LM NLL and detector score is
$\rho=-0.76$).
% S1.284
The effect is typicality rather than mere formality: in a
temperature$\times$register factorial at matched length, typicality dominates
formality by $2.6\times$, $8.6\times$, and $3.4\times$ across detectors
($n=102$).
% S1.232
This earns ``typicality-\emph{dominant}'': formality leaves a small,
inconsistent residual, so the axis is not authorship-blind, but it is
typicality that does the work; the high-temperature AI cleared in
Figure~\ref{fig:inversion} is the same effect.

Editing a human text toward typicality raises its AI-score. The
human-side leg is causal too. A minimal, content-preserving grammar
correction, roughly length-stable, raises a non-native essay's $\paff$ by
$+0.15$ to $+0.21$, about three times the rate for already-fluent native
writing.
% S1.290
A non-neural corruption control and a clean-text null (edits to already-clean
NYT move the score by $\approx 0$) establish that it is the gain in
typicality, not the act of editing, that flags the text. The relationship is
bounded to coherent edits: heavy corruption or very high temperature makes
text atypical in another way and re-flags it as AI.

Absolute-likelihood scores over-flag; curvature escapes, and we can
say why. Scores that read absolute predictability rank formal humans as more
machine-like than informal humans; curvature scores that normalize by the
local density do not. The over-flag AUROC of formal against informal humans
is $0.675$ for reference-LM NLL, $0.907$ for log-rank, and $0.870$ for GLTR
top-$1$ ($0.5$ marks no preference), whereas Fast-DetectGPT and canonical Binoculars sit below it at
$0.328$ and $0.454$ respectively.
% S1.253 (over-flag AUROC = AUROC(NYT-formal vs HC3-informal); >0.5 over-flags, <0.5 escapes; FDG 0.328 / canonical Binoc 0.454 escape)
The reason is that an optimal test computes a likelihood ratio, while an
absolute-perplexity score keeps only a stand-in for its numerator and drops
the human-density denominator:
\begin{equation}
\Lambda(x) = \frac{p(x \mid \text{AI})}{p(x \mid \text{human})}
\;\longrightarrow\;
\tilde{\Lambda}(x) \propto p(x \mid \text{ref}),
\label{eq:degenlr}
\end{equation}
so it ranks text by typicality-under-the-reference-model rather than by
evidence of machine authorship, and fluent, formal human writing is exactly
the high-typicality text that such a degenerate test misreads. This is not merely an
interpretation: grading the score from the absolute likelihood back to the
density-normalized form, which restores the dropped denominator, drives the
formal-human over-flag monotonically to zero (Appendix~\ref{app:degenlr}).
% L1 (agent2_RESULTS): FDG interpolation Spearman 1.000; Ghost gap -0.037->+0.041; literal route confirms

That degenerate test reads an external reference model; the fine-tuned
detector instead reads typicality internally. Its over-flag tracks the
detector's \emph{internal} typicality almost exactly (the length-matched
projection-to-score correlation is $\rho=-0.94$), whereas an external GPT-2 NLL collapses to
non-significance once length is matched.
% S1.274
We therefore call the axis typicality but treat any external perplexity as a
lossy proxy, never a literal ``perplexity meter.''

The over-flag is register, not training-set membership. Held-out HC3
humans are over-flagged at $0.049$, indistinguishable from the
in-training rate of $0.045$, and even scripted dialogue, unseen but highly
predictable, is over-flagged at $0.723$.
% S1.277
The axis reads predictability, not familiarity.

\subsection{The price and the no-go}
\label{sec:nogo}
We now reach the claim the paper turns on: the axis that transfers across
generators is the axis that over-flags formal humans, so the unfairness
cannot be removed without giving up cross-generator detection.

\begin{figure}[t]
  \centering
  \includegraphics[width=\columnwidth]{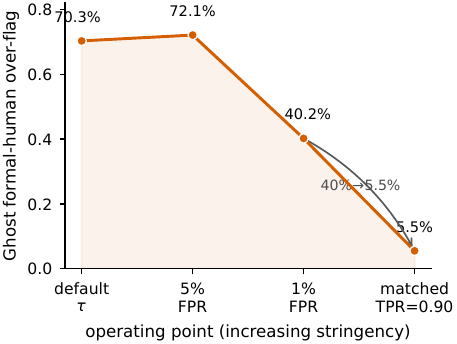}
  \caption{The deployment frontier. Same ELECTRA detector, threshold only: the
    Ghost formal-human over-flag falls from $40\%$ at $1\%$ FPR to $5.5\%$ at
    matched-TPR$=0.90$. The harm is an operating-point phenomenon, large where
    detectors actually ship.}
  \label{fig:frontier}
\end{figure}

The over-flag is an operating-point phenomenon, and structural where it
counts (Figure~\ref{fig:frontier}). How large the harm is depends entirely
on the threshold: the same ELECTRA detector over-flags formal-native essays
(Ghost) at $40\%$ at a $1\%$ false-alarm rate but only $5.5\%$ at matched-TPR.
% S1.267
This rebuts the ``just a calibration bug'' reading twice over. First, $40\%$
is the regime detectors deploy in, and the over-flag is paradigm-wide at the
default threshold (cross-architecture $0.70$/$0.67$/$0.73$ for E/R/D, the
tightest band in the paper).
% S1.282
Second, the part that survives recalibration is real: on Ghost the over-flag is length-robust and
structural, whereas NYT's is partly a length effect that recalibration clears.
% S1.276

\begin{figure}[t]
  \centering
  \includegraphics[width=\columnwidth]{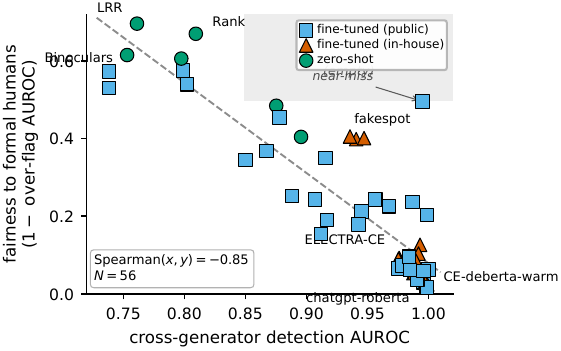}
  \caption{Fairness$\leftrightarrow$generalization frontier across 56 deployed
    detectors (20 detectors weak on both axes excluded: cross-generator AUROC
    $\le0.85$ \emph{and} fairness $\le0.5$). Spearman $-0.85$; the
    fair-and-general region is empty, and the closest near-miss
    (\emph{fakespot}) sits at the boundary at $(0.995, 0.4955)$. Harm is
    sharpest for public detectors.}
  \label{fig:tradeoff}
\end{figure}

Across deployed detectors, fairness to formal humans and cross-generator
robustness do not co-exist (Figure~\ref{fig:tradeoff}): the
fair-and-general region is empty across 56 detectors, and the closest near-miss,
\emph{fakespot}, sits at the boundary at $(0.995, 0.4955)$. Fairness and
cross-generator detection trade off at Spearman $-0.85$; because fairness here
is just a low over-flag, the over-flag and the cross-generator signal are two
readings of one axis, typicality at once the signal and the bias.
% S1.279 + Fig-4 expansion ROUND 3 2026-05-26: N=56 curated-strongest from the principled set of 76 (FIG4_EXPANSION_RESULTS_2026-05-26_round3.md); 20 detectors weak on both axes (x<=0.85 AND fairness<=0.5) excluded as uninformative about the trade-off. Top-right (fair+general) empty across all 76. Full-honest set N=76, Spearman -0.78. Round-2 N=20/-0.88 backup in figs/tradeoff_BACKUP_round2_N20.pdf.
The harm is sharpest for \emph{public} detectors, since an in-house detector
can recalibrate on its own formal-human labels.

The transferring axis is the formal-harm axis, and the harm is
localized. On the warmstart representation whose typicality reading transfers
across generators (AUROC $0.83$), the formal-human false-positive floor is an
order of magnitude above the informal one (Ghost $0.103$ versus informal
$0.011$); on the deployed head the gap persists, Ghost $0.089$ against
HC3 $0.017$, close to the informal-human Bayes floor ($\le 2.4\%$).
% S1.287
The harm is targeted, not detector incompetence: where register matches,
detection is near-perfect.
% S1.293
This is why the bias resists removal. Standard debiasing assumes the
offending feature is a spurious correlate to be discarded; here it is the
\emph{only} feature that generalizes. The invariant is the confound.
No intervention we test removes the harm while preserving
cross-generator detection, and the failure is mechanical, not incidental.
To catch $90\%$ of held-out, novel-generator AI, an ELECTRA detector must
flag $69\%$ of cross-corpus formal human writing, far more than at the
in-distribution matched-TPR of Figure~\ref{fig:frontier} because catching
novel generators forces a much lower threshold; same-domain the cost is far
milder ($0.13$), so the no-go is the cross-corpus, novel-generator
regime, not a claim about every population.
% S1.288 (BL#4: EvoBench + modern-commercial, not NYT-news-AI 0.02)
The failure is Simpson-style: a fairness penalty zeros the
\emph{within}-population score-typicality correlation, but the over-flag is a
between-population level shift. The formal-human logit gap thus moves the
wrong way ($-0.27$ to $+1.67$): the penalty acts on the within-group slope
while the harm lives in the between-group level.
% S1.266 (wrong-lever)

\begin{table*}[t]
  \centering
  \small
  \setlength{\tabcolsep}{4.5pt}
  \begin{tabular}{l l c l}
    \toprule
    Intervention family & Formal-human over-flag & ID-AUROC & Cross-gen.\ transfer \\
    \midrule
    \multicolumn{4}{l}{\emph{Mode 1: transfer held, over-flag not removed (AUROC-neutral relocation)}}\\
    Inference $-\varepsilon$         & $0.055\to0.047$       & $0.999$ (flat) & held ($\Delta\le0.005$) \\
    LEACE (closed-form erasure)      & $+0.107$              & held           & held (axis nulled $5.0\to6\times10^{-6}$) \\
    Invariance (IRM/V-REx/CORAL)     & $0.046\to0.040$ (E)$^{\dagger}$       & $0.999$        & held (xferEvo $0.97$) \\
    \midrule
    \multicolumn{4}{l}{\emph{Mode 2: transfer eroded, over-flag still not removed}}\\
    Correlation penalty (pooled)     & $0.40\to0.98$ (R)     & held           & collapses (${\sim}0.9\to0.04$--$0.21$) \\
    Correlation penalty (stratified) & $0.047/0.41/0.191$ (worst) & held (R $\to0.89$) & erodes (R $0.88\to0.05$) \\
    Adversarial (GRL)                & $0.21\to0.32$ (NYT, $\tau_{\mathrm{def}}$) & $0.96\to0.89$  & destroyed ($0.76\to0.16$) \\
    Training-time projection         & $0.002\to0.025$ (NYT) & $0.9997$       & held; probe recovers axis \\
    \midrule
    \emph{FT decoder (relocation)}   & relocates             & n/a            & AI-vs-FCE $0.993\to0.835$ \\
    \bottomrule
  \end{tabular}
  \caption{The unified no-go. Six intervention families across three
    architectures (ELECTRA/RoBERTa/DeBERTa) and a fine-tuned decoder: none lowers
    the formal-human over-flag while holding in-distribution AUROC and
    cross-generator transfer. Over-flag is the Ghost rate at matched-TPR$_{0.90}$
    on ELECTRA unless tagged: R/D $=$ RoBERTa/DeBERTa, NYT $=$ NYT population,
    $\tau_{\mathrm{def}}=$ default threshold, worst $=$ worse of Ghost/NYT; a
    three-value cell is E/R/D (Ghost at matched-TPR was not measured for the
    adversarial and projection arms). Two modes:
    inference $-\varepsilon$ and LEACE hold transfer but only relocate the
    over-flag (AUROC-neutral); correlation, adversarial, and projection penalties
    erode or destroy transfer and still do not debias.
    $^{\dagger}$The IRM/V-REx/CORAL row's small E-only decrease is offset on
    R (transfer collapses $0.87\to0.10$) and D (NYT relocation
    $0.34\to0.77$), Appendix~\ref{app:nogo}.
    Scoped to EvoBench $+$ modern-commercial generators;
    full panel in Appendix~\ref{app:nogo}.}
  \label{tab:nogo}
\end{table*}

Across six intervention families and three architectures
(Table~\ref{tab:nogo}) the result holds, in two distinct modes. In the first, our $-\varepsilon$ operator and LEACE hold cross-generator
transfer but do not debias: LEACE provably nulls the linear axis
($5.0\to6\times10^{-6}$), yet the Ghost over-flag at matched-TPR rises by
$+0.107$. In the second, correlation and adversarial penalties erode transfer
instead (a gradient-reversal detector collapses cross-LM detection from $0.76$
to $0.16$), and even a within-population-stratified penalty, immune to the
Simpson cancellation that defeats the pooled one, stays non-corrective.
% S1.250 (LEACE) ; S1.266 (stratified pooled|corr|->0.0001 vs 0.9576) ; S1.268 (DANN)
Objectives built to force cross-generator invariance fare no better: IRM,
V-REx, CORAL, and cross-generator contrastive training do not remove the
over-flag without paying cross-generator detection or relocating the harm to
another population (three architectures, three seeds), because the
cross-generator invariant they enforce is the typicality confound
(Appendix~\ref{app:nogo}).
% H4 (2026-05-25_claude_H2_H4_RESULTS): ELECTRA 0.91->0.88 hold xfer; R collapse 0.87->0.10 xLM; D trade 0.687->0.626 + NYT 0.34->0.77; 3-arch 3-seed
No ensemble escapes either, the result extends into fine-tuned decoders, and the
coupling is scale-invariant to 8B (Appendix~\ref{app:nogo}). The over-flag
is operating-point-relocatable but, for any method that must generalize to
unseen generators, surgically un-removable.
% S1.289 (ensemble 0.198/0.076) ; S1.273 (decoder 0.993->0.835 in Table 1) ; S1.266 (8B scale-inv) -- numbers in app/table
%
The no-go has the structure of a theorem. The empirics are not a list
of recipes that happened to fail; they instantiate three results, proved in
Appendix~\ref{app:theory}, that follow from the rank-one geometry and foreclose
every on-axis fix. A population-blind recalibration can relocate the
formal-human over-flag and trade it against missed detections, but cannot
equalize it (Proposition~\ref{prop:recal}); at fixed detection, any detector
that reads only the typicality axis has its over-flag frozen at that
population's overlap with AI on the axis (Proposition~\ref{prop:frontier}); and
escaping that floor requires an off-axis move, which is then necessary and
sufficient to debias while preserving in-distribution detection
(Proposition~\ref{prop:dichotomy}). The signatures are empirical:
recalibration-invariance (eighteen training variants in $[0.026,0.150]$,
cross-seed spread $0.018$ versus cross-method $0.386$), the frozen matched-TPR
over-flag numbers, and the off-axis revival of \S\ref{sec:tool}. The result is a
detection-specific instance, in the spirit of the fairness-impossibility
theorems \citep{chouldechova2017,kleinberg2017}, with one twist: what plays the
role of the protected attribute, how typical a writer's prose is, is not merely
correlated with the decision statistic, it \emph{is} the decision statistic,
which is why no post-hoc reconciliation exists.
% S1.267 (18 variants [0.026,0.150]; seed 0.018/method 0.386) ; S1.287 (R3) ; S1.269/S1.268 (R4). Prop 1-3 statements moved to App A.

Where the price is paid, precisely. A careful reader will note that
our own $-\varepsilon$ operator (\S\ref{sec:tool}) removes the typicality
direction at inference and is AUROC-neutral; this is not a self-refutation.
Removing the direction from an already-fine-tuned detector \emph{is}
detection-neutral, it relocates false positives without changing
separability.
% S1.270 (in-FT erasure detection-neutral)
The price appears instead in two places the no-go identifies: at
training time, where every effective typicality-removal collapses
cross-generator transfer (from $0.94$ to $0.36$ as the penalty strengthens,
while in-distribution AUROC moves only $0.999\to0.995$) with no fairness
dividend;
% S1.268 (the dial)
and in the decomposition, where the only sub-component that debiases is the
residual, the one that does not transfer.
% S1.287

\subsection{The tool: training-free, transferable, honest}
\label{sec:tool}
The decomposition is constructive: the same split that forecloses a cure
yields a training-free tool, a closed-form operator that predicts its own
effect. The $-\varepsilon$ operator of \S\ref{sec:methods} removes the
typicality component, and its per-text predictor is
accurate enough to act as an oracle: under a full composition swap its
chosen action matches the post-hoc optimum on $35$ of $36$ cells, and a
signed-$\varepsilon$ sweep is strict-Pareto on all $18$ configurations and
Pareto on all $27$ (predictor MAE $0.005$).
% S1.131 ; S1.111 / S1.125
The predictor is exact on linear and decoder heads and Taylor-accurate
($\varepsilon\le0.7$) on encoders.

The tool transfers to detectors we did not build, and beats published
debiasers. Because the predictor reads only a representation and a direction,
it diagnoses bias zero-shot on detectors trained by others ($R^2\ge0.987$).
% S1.275
Head to head it dominates the standard concept-erasure remedies: a
$25\times$ lower false-positive rate than INLP at matched recall (true-positive
rates within $\pm0.02$), and none of
RLACE's $18$ configurations converged. Against the zero-shot Binoculars foil,
the supervised representation catches sixfold more held-out-generator AI at a
matched false-alarm rate (true-positive rate $0.966$ versus $0.159$).
% S1.115

\begin{figure}[t]
  \centering
  \includegraphics[width=\columnwidth]{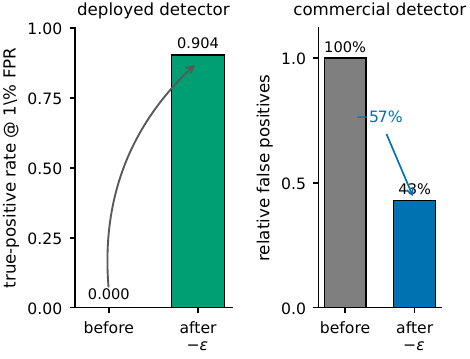}
  \caption{Dead-detector revival. At a $1\%$ false-alarm rate on formal NYT
    humans a shipped detector catches no AI (true-positive rate $0$), because
    formal humans saturate its score alongside AI; a closed-form, training-free
    axis ablation separates them, lifting the rate to $0.904$ (AI-vs-NYT AUROC
    $0.955\to0.991$ on this deployed detector). This relocates the operating
    point; cross-generator separability is unchanged (\S\ref{sec:tool}).}
  \label{fig:revival}
\end{figure}

The headline practical win is reviving a dead deployed detector: one that
catches nothing at a $1\%$ false-alarm rate has its true-positive rate lifted
from $0$ to $0.904$ (Figure~\ref{fig:revival}). On a commercial detector it
cuts false positives by $57\%$, and at the default threshold it cuts the NYT
over-flag from $0.374$ to $0.252$.
% S1.94 ; S1.186 ; S1.278 (default-tau NYT calibration, rho=-1.00)

What it does not do is erase the bias. Consistent with the no-go, removing
the axis leaves the detector's cross-generator separability unchanged
(AUROC-neutral): it relocates false positives along the operating curve rather
than removing them. It is doubly bounded: on 2025-commercial generators it
\emph{costs} detection at the operating point (the true-positive rate at a
$1\%$ false-alarm rate falls $0.818\to0.741$, and the Ghost over-flag rises),
and it is encoder-paradigm-specific, null on fine-tuned decoders whose head
decides off the typicality axis.
% S1.257 / S1.278 / S1.272
On its calibration leg the operator joins recalibration and group-DRO as one
relocation class; the leverage law and the encoder-versus-decoder mechanism are
in Appendix~\ref{app:predictor}.
% S1.237
It is, by construction, a relocator, not a cure.

Finally, the operator is not specific to AI detection: the same per-text
ablation reduces occupational-gender bias on the Bios task across all three
architectures (strict-Pareto), with further dialect and content-moderation
transfers in Appendix~\ref{app:crosstask}.
% S1.139 (Bios) ; S1.135 / S1.27 (appendix; AAVE = transfer demo only)

\subsection{Predictable, structural fairness}
\label{sec:fairness}
Because the harm rides one axis, it is forecastable, per population,
per detector, without access to the detector's internals. A population's
exposure is predicted by how typical its writing is: a cosine-to-AI
diagnostic predicts the per-population over-flag ordering at a $1\%$ false-alarm
rate (Spearman $0.986$) across seven populations, on detectors it never saw. This diagnostic is
not the decision statistic re-read: the typicality axis is near-orthogonal to
the probe-derived linear approximation of the decision (cosine $0.103$;
App.~\ref{app:fairness}).
% S1.280 / S1.275 / S1.274 (circularity)
The full ten-population matched-TPR panel is in Appendix~\ref{app:fairness}.

The over-flag is common across deployed fine-tuned detectors, a
property of the class, not one bad system. Four of six deployed fine-tuned
detectors over-flag formal writing at matched length (chatgpt-detector-roberta
$0.97$, our ELECTRA $0.91$, the OpenAI RoBERTa detectors
\citep{solaiman2019} openai-large $0.72$ and openai-base $0.74$); the two
exceptions are a length artifact (RADAR; \citealp{hu2023radar}) and a
no-over-flag detector, so the property is common rather than universal.
% S1.280 (BL#8: 4/6 common, not universal; openai-base 0.74 is register/length-controlled)

The harm falls on formal-native fluent writing, the one case no
intervention protects. Naming the harmed group precisely also pre-empts the
reading that this is anti-non-native bias. Non-native English (FCE) is cleanly
separable from out-of-distribution AI (AUROC $0.95$--$0.99$) and is
protectable across populations \emph{without any FCE labels} (its over-flag
drops to $\le 0.13$ on RoBERTa and DeBERTa even at the novel-generator operating point), whereas formal-native essays
(Ghost) stay over-flagged on every architecture, the surgically un-removable
case.
% S1.292
The floor on dialect text (CORAAL, Hinglish, FCE at matched-TPR) is near
zero, so the harm is anti-formal-prose, not anti-non-standard-English.
% S1.266
Nor can the harmed writer easily escape. Clearing the flag means writing less
typically: on the internal axis a formal-native essay must be made more
atypical than up to $76\%$ of natural humans to cross the threshold (versus
$10$--$29\%$ for journalism). On the strongest detectors this cost exceeds the
coherent-rewriting regime, and much over-flagged formal writing cannot clear at
all (Appendix~\ref{app:fairness}).
% L5+RESCUE (2026-05-25_agent2_L4L5L6): N* internal Ghost 49-76% vs NYT 10-29%, Ghost>NYT 3/3; gen-invariant; non-circular

The axis says ``this is AI,'' not which model produced it. Consistent
with the decomposition, generator identity lives in the residual rather than
the transferring axis: typicality attributes the generator only at the random
one-dimensional floor ($1.4\times$ chance, against $6\times$ for the full
representation). The residual carries the modest, above-chance generator
signal, which we develop in Appendix~\ref{app:crosstask}.
% S1.291 (BL#10: modest above-chance, not a strong fingerprint)

\section{Discussion}
\label{sec:discussion}
The blind spot is the whole ecosystem, and it scales. Shopping among
detectors is false comfort: eight detectors and a zero-shot perplexity
detector share a single axis (its first principal component explains $32\%$ of
the variance), and no ensemble over them escapes (\S\ref{sec:nogo}).
% S1.289
The typicality coupling is scale-invariant up to 8B parameters ($-0.24$ at 8B
$\approx -0.245$ at 110M), so this is a property of the detection paradigm,
not an artifact of small encoders. Nor is it specific to English prose: the
over-flag and the inherited-residual decomposition replicate in a typologically
distant language (Chinese) and in a non-natural-language domain (Python code;
Appendix~\ref{app:xling}).
% S1.266 ; xling H1/H3 (2026-05-25_agent1_H1_H3_xling_code_RESULTS)

Honest mitigations relocate the harm; they do not cure it.
In-population calibration helps the population it is tuned for, it cuts the
Ghost over-flag from $0.68$ to $0.42$, but it relocates the harm; under aggressive supervision, another population's over-flag can be
driven as high as $0.80$.
% S1.292
The $-\varepsilon$ operator, recalibration, and group-DRO are one relocation
class (\S\ref{sec:tool}); the only escape we find is explicit formal-human
supervision, which trades the problem for a labeling burden on every
population.

The harm is era-invariant. Because it is driven by the typicality of
the human reference rather than by the generator, it will not fade as
generators improve: the matched-protocol over-flag is unchanged across eras
($0.986$ versus $0.974$), and across a 14-year span of human reference text at
matched formality (Appendix~\ref{app:fairness}).
% S1.251 (BL#6: no era trend) ; L6 human-era (2026-05-25_agent2_L4L5L6)
Detectors trained and deployed today will keep penalizing typical, formal
human writing regardless of how models evolve.

% ===========================================================================
% LIMITATIONS  (MANDATORY -- desk-reject if missing; unnumbered). ~0.75pp.
% One line each; none weakens a scoped headline. De-hedged per audit.
% ===========================================================================
\section*{Limitations}

\paragraph{Scope of ``read, not learned.''}
Our claim concerns the \emph{bias} axis, not fine-tuning's total
contribution: a matched-perplexity authorship residual remains, so typicality
is \emph{dominant}, not exclusive. (This residual also absorbs the
RoBERTa-base geometric flip, on which the $-\varepsilon$ operator raises the
over-flag rather than lowering it.) ``Read, not learned'' therefore means an
inherited representation, a non-generalizing residual, and label-efficiency,
not that a frozen probe beats fine-tuning, which it does not (clean transfer:
frozen $0.752 <$ fine-tuned $0.831$).

\paragraph{Method scope.}
The label-efficient probe is single-backbone-fragile; robust deployment needs
a multi-backbone max-pool. The closed-form predictor is exact only on linear
and decoder heads; on fine-tuned encoders it is first-order accurate for
$\varepsilon\le0.7$.

\paragraph{Scope of the no-go, and anchored magnitudes.}
The no-go is the cross-corpus, novel-generator regime: within domain, AI and
formal humans still separate and the over-flag is milder ($0.13$--$0.23$).
In-population calibration is genuine but partial, and relocates harm. Finally,
magnitudes are anchored: to the training recipe (HC3, within-recipe; encoders
and decoders warm-started, RoBERTa raw-base, 8B a different substrate) and to
the generator era (the raw axis falls toward chance on modern anchors). The
era-anchoring is reviewer-checkable and is the same reference-formality
dependence that makes the harm era-invariant (\S\ref{sec:discussion}).

\bibliography{custom}

\appendix

\section{Theory: premises and proofs}
\label{app:theory}

We make the claims of \S\ref{sec:nogo} precise. All results are population-level,
scope the typicality axis to the \emph{bias readout} (dominance, not
exclusivity), and treat the rank-one premise as an explicit conditional.

\paragraph{Notation.}
For a text $x$ write $\cls=\cls(x)\in\mathbb{R}^{n}$ for its representation,
$\dtyp$ the unit typicality axis, $s=\langle\cls,\dtyp\rangle$ its projection,
and $\cls_\perp=\cls-s\,\dtyp$ the off-axis component (so
$\cls=s\,\dtyp+\cls_\perp$, $\cls_\perp\perp\dtyp$). For a class $Q$, let
$L_Q(s)$, $L_Q(\cls_\perp)$ be the laws of $s$, $\cls_\perp$ under $Q$. For a
threshold $\theta$ the survival function is $F_i(\theta)=\Pr_{x\sim P_i}[s\ge\theta]$
and $G(\theta)=\Pr_{\mathrm{AI}}[s\ge\theta]$. Write the AI--population total
variations on and off the axis as
$\delta_i^{s}=\mathrm{TV}(L_{\mathrm{AI}}(s),L_{P_i}(s))$ and
$\delta_i^{\perp}=\mathrm{TV}(L_{\mathrm{AI}}(\cls_\perp),L_{P_i}(\cls_\perp))$.
$P_+$ is an operating-point-robust typical population (Ghost), $P_-$ the
atypical anchor (HC3).

\paragraph{Premises (each a conditional with its empirical anchor).}
\begin{description}
\item[\textnormal{(B1) Rank-one bias readout.}] For human text the deployed
score is a population-independent, non-decreasing function of the typicality
coordinate, $f(x)=\varphi(s(x))$; equivalently, conditioned on $s$ the score
law is common across human populations, and the flag probability is $a(s(x))$
with $a$ non-decreasing. \emph{Anchor:} the bias
decision rank is $k^{*}{=}1$ in-distribution rising to $3$--$10$ out of
distribution, and a single head direction reproduces the fine-tuned AUROC
within $0.02$ on $36/36$ cells (Appendix~\ref{app:nogo}); we support B1 by this
OOD-rank contrast, not by a PC1--logit $R^2$.
\item[\textnormal{(B2) Typicality separation.}] $F_+(\theta)\ge F_-(\theta)$
for all $\theta$, strict on the deployment range. \emph{Anchor:} typical
writing has higher $s$ and FPR at matched length; $P_+$ over-flags above $P_-$
at every operating point on the range.
\item[\textnormal{(B3) AI--$P_+$ overlap on the axis.}] $\delta_+^{s}$ is small
enough that the over-flag persists at useful detection. \emph{Anchor:} Ghost
FPR at matched-TPR$=0.90$ is $0.056/0.176/0.436$ (E/D/R), strictly positive;
the deployed detector has TPR $0$ at $1\%$ FPR before intervention.
\item[\textnormal{(B4) Separable off-axis detection signal.}]
$\delta_+^{\perp}$ is large on a fixed generator; \S\ref{sec:read} shows this
off-axis signal does not transfer to unseen generators. \emph{Anchor:}
residualizing the logit on the raw representation leaves AUROC $0.82/0.63/0.86$
(RoBERTa's $0.63$ the boundary case); the off-axis detector reaches
AI-versus-Ghost AUROC $0.87$--$0.97$, and the $-\varepsilon$ ablation realizes
$(\mathrm{FPR}_+,\mathrm{TPR})=(0.01,0.904)$.
\item[\textnormal{(R) Recalibration class.}] A population-blind post-hoc rule:
it sees $x$ only through the deployed score $f(x)$ and applies a non-decreasing
transform and threshold (optionally randomized between thresholds).
\end{description}

\begin{proposition}[Recalibration no-go]
\label{prop:recal}
Suppose that across human populations the flag decision factors through the
single typicality coordinate $s(x)=\langle\cls(x),\dtyp\rangle$ (the
rank-one bias premise). Then at any fixed detection rate the per-population
false-positive profile is determined: a population-blind recalibration can
\emph{relocate} the over-flag across populations and \emph{trade} it against
missed detections, but cannot equalize it (the only recalibration that
equalizes is the non-discriminating chance rule). Equalizing false positives at
fixed detection requires an intervention that does not factor through $s$.
\end{proposition}

\paragraph{Proof of Proposition~\ref{prop:recal} (recalibration no-go).}
\emph{(i)~Reduction.} A deterministic recalibration flags iff $T(f(x))\ge\tau$
with $T$ non-decreasing. By (B1) the deployed score on human text is monotone
in $s$ with a population-independent relation, so $f=\varphi(s)$ for
non-decreasing $\varphi$; then $\{s:T(\varphi(s))\ge\tau\}=[\theta,\infty)$ and
$\mathrm{FPR}_i=F_i(\theta)$ (a randomized rule gives $\int F_i\,d\mu$).
\emph{(ii)~Fixed-detection invariance.} Fixing
$\Pr_{\mathrm{AI}}[f\ge\tau]=t\in(0,1)$, where the AI detection rate is strictly
decreasing at $t$ (atomless, no support gap), pins $\tau$, hence $\theta$, hence
every $F_i(\theta)$: the matched-detection FPR
profile is a recalibration invariant. \emph{(iii)~Ordering sign-lock.} Adding
(B2), $\mathrm{FPR}_+-\mathrm{FPR}_-=\int(F_+-F_-)\,d\mu\ge0$, strict at any
threshold partially flagging both populations; equality forces the
non-discriminating chance rule, so no recalibration equalizes $P_+,P_-$ while
detecting AI above chance. \emph{(iv)~Off-axis necessity.} By (ii), lowering a
matched-detection FPR is not a recalibration; since the human flag is a
threshold on $s$ (B1), it requires changing $x\mapsto s(x)$, an off-axis
intervention, feasible by (B4). \qed

\begin{corollary}[Off-axis operator]
\label{cor:operator}
The intervention of part~(iv) is realized exactly on a linear head: removing
the axis component, $\cls\mapsto\cls-\varepsilon\,s\,\dtyp$, changes the logit
by exactly $\Delta\mathrm{logit}=-\varepsilon\,s\,\langle\wh,\dtyp\rangle$
(Eq.~\ref{eq:predictor}). It does not factor through any threshold on $s$ and
leaves $\cls_\perp$, hence the separable signal of (B4), intact. It is exact
for linear and decoder score heads and first-order accurate
($\varepsilon\le0.7$) on fine-tuned encoders.
\end{corollary}

\begin{proposition}[Matched-detection frontier]
\label{prop:frontier}
At fixed detection, the achievable per-population false-positive profile is a
single point on the populations' typicality-axis probability--probability
curve $C(u)=F_{+}(F_{-}^{-1}(u))$. The typical/atypical gap is frozen at
$C(u)-u\ge 0$, and for any detector that reads only the axis, the typical
population's over-flag cannot fall below $t-\delta_{+}^{s}$, its total-variation
overlap with AI on the axis.
Recalibration slides along this curve; only an off-axis move reaches a lower
one.
\end{proposition}

\paragraph{Proof of Proposition~\ref{prop:frontier} (matched-detection
frontier).}
\emph{(i)~Frontier as a P--P curve.} By (i) above each human accept region is
$\{s\ge\theta\}$, so $\mathrm{FPR}_i=F_i(\theta)$; eliminating $\theta$ gives
$\mathrm{FPR}_+=C(\mathrm{FPR}_-)$ with $C=F_+\circ F_-^{-1}$, the
probability--probability plot of $P_+$ against $P_-$ on $s$. $C$ is
non-decreasing, $C(0)=0$, $C(1)=1$, and randomized rules fill its convex hull;
under (B2), $C(u)\ge u$, so the gap $\Gamma(u)=C(u)-u\ge0$. \emph{(ii)~Matched
detection freezes the point.} As in part~(ii) above, fixing $T=t$ pins
$\theta_t$, hence the single point $(u_t,C(u_t))$, $u_t=F_-(\theta_t)$; the gap
$\Gamma(u_t)$ is recalibration-invariant. \emph{(iii)~Over-flag floor.} For the
axis-only meter $T(\theta)=G(\theta)$, so $G(\theta_t)=t$; the single-event
bound $|\mu(E)-\nu(E)|\le\mathrm{TV}(\mu,\nu)$ at $E=\{s\ge\theta_t\}$ gives
$F_+(\theta_t)\ge t-\delta_+^{s}$. \emph{(iv)~Off the curve.} Lowering
$\mathrm{FPR}_+$ at fixed $t$ leaves the recalibration class, moving the
detector to a strictly lower curve $C'$ by (B4)
(Corollary~\ref{cor:operator}). \qed

\paragraph{Calibration-fixable versus structural.}
Modelling the in-distribution statistic as $s\mid P\sim\mathcal{N}(\mu_P,\sigma^2)$
(justified by in-distribution $k^{*}{=}1$) makes the frontier closed-form,
$\mathrm{FPR}_+=\bar\Phi(\bar\Phi^{-1}(\mathrm{FPR}_-)-\Delta)$ with
$\Delta=(\mu_+-\mu_-)/\sigma$. Inverting the Ghost matched-TPR residuals yields
effective separations $\Delta_{\mathrm{Ghost}}\approx2.87/2.21/1.44$ (E/D/R),
all in the overlap (structural) regime, against NYT's $\Delta\approx4.16$
(separated, calibration-fixable). This location-shift calibration is
illustrative; the rigorous content is the non-parametric frontier and the
floor.

\begin{proposition}[Separability dichotomy]
\label{prop:dichotomy}
If AI and a typical human population nearly coincide on the axis
($\delta_{+}^{s}$ small) but separate off it ($\delta_{+}^{\perp}$ large),
then every population-blind function of the axis obeys
$\mathrm{TPR}-\mathrm{FPR}_{+}\le\delta_{+}^{s}$, it cannot detect AI
without over-flagging the typical population, whereas removing the axis
component routes the decision onto the off-axis signal. An off-axis
intervention is therefore necessary and sufficient to debias while preserving
in-distribution detection.
\end{proposition}

\paragraph{Proof of Proposition~\ref{prop:dichotomy} (separability dichotomy).}
Assume (B1), (B3), (B4). \emph{(i)~Axis-only floor.} By the variational
identity $\mathrm{TV}(\mu,\nu)=\sup_{B}[\mu(B)-\nu(B)]$, for any measurable
accept set $S\subseteq\mathbb{R}$ (any population-blind function of $s$,
recalibration being the half-line case),
$\mathrm{TPR}-\mathrm{FPR}_+=L_{\mathrm{AI}}(S)-L_{P_+}(S)\le\delta_+^{s}$.
\emph{(ii)~Necessity.} Attaining $\mathrm{FPR}_+<\mathrm{TPR}-\delta_+^{s}$ is
therefore impossible for any function of $s$; the rule must read $\cls_\perp$.
\emph{(iii)~Sufficiency.} The ablation at full strength sends the linear-head
logit to $\langle\wh,\cls_\perp\rangle+b$ (the $\varepsilon{=}1$ reading of
Eq.~\ref{eq:predictor}; the constant $b$ is immaterial to the threshold); a detector on $\cls_\perp$ separates AI from $P_+$ up
to $\delta_+^{\perp}\ge\mathrm{TPR}-\mathrm{FPR}_+$, realized by (B4):
AI-versus-Ghost AUROC $0.87$--$0.97$ gives
$\delta_{\mathrm{Ghost}}^{\perp}\ge\mathrm{AUROC}-\tfrac12\ge 0.37$, and the
operating point $(0.01,0.904)$ on NYT gives
$\delta_{\mathrm{NYT}}^{\perp}\ge 0.894$; either way
$\delta_+^{\perp}\gg\delta_+^{s}$. \emph{(iv)~Unification.}
Recalibration (Proposition~\ref{prop:recal}) is the monotone half-line case of
(i); Corollary~\ref{cor:operator} is the sufficient off-axis rule of (iii).
On-axis is floored at $\delta_+^{s}$; off-axis escapes to $\delta_+^{\perp}$. \qed

\begin{proposition}[Pooled-penalty Simpson no-go]
\label{prop:simpson}
Let groups $g$ (population $\times$ label) partition the training mixture and a
penalty drive the pooled correlation $\mathrm{corr}(f,\mathrm{NLL})$ to zero. If
every within-group coupling stays non-negative (C1), the pooled correlation can
vanish only by between-group cancellation of the within-group term, leaving the
within-group couplings and the formal-population score level intact; and the
cancellation is mixture-specific, so re-weighting the groups (a new generator
distribution) makes the pooled coupling re-emerge.
\end{proposition}

\begin{proof}
By the law of total covariance,
$\mathrm{cov}(f,\mathrm{NLL})=\mathbb{E}_g[\mathrm{cov}(f,\mathrm{NLL}\mid g)]+\mathrm{cov}_g(\mathbb{E}[f\mid g],\mathbb{E}[\mathrm{NLL}\mid g])=:W+B$.
By (C1), $W\ge0$; $W+B=0$ forces $B=-W\le0$. No conditional law $(f\mid g)$,
in particular no conditional mean $\mathbb{E}[f\mid g]$ (the population level,
i.e.\ the bias), is constrained beyond supplying $B$. Re-weighting with
$\{\pi'_g\}$ gives $W'=\sum_g\pi'_g\,\mathrm{cov}(f,\mathrm{NLL}\mid g)\ge0$,
still strictly positive on any positively-coupled group, while $B'\ne-W'$ in
general; hence $\mathrm{cov}'_{\mathrm{pool}}\ne0$ generically. The bias is
untouched and the decorrelation does not survive a generator shift.
\end{proof}

\noindent\emph{(C1) Within-group couplings stay non-negative under the
penalty.} Anchor: at penalty weight $\lambda{=}3$ the within-population couplings of
detector score and reference-LM NLL are positive ($+0.12/+0.36/+0.74$ on
NYT/HC3/AI), the penalty having inverted their natural negative sign; the
pooled $|\mathrm{corr}|\to0.03$ and the formal-population level barely moves. The
general ``any objective that removes $s$-dependence must de-generalize'' claim
is \emph{not} a theorem (it needs a circular premise), and the inference-time
operator is a live counterexample (it removes $s$-dependence yet keeps
detection); the no-go is therefore stated for pooled training-time penalties
only.

\section{No-go breadth}
\label{app:nogo}
This appendix expands the intervention families of Table~\ref{tab:nogo} across
architectures and held-out generators, lists the legacy null interventions, and
states the rank-one premise the theory rests on.

\paragraph{Transfer collapse across penalty families.}
Every penalty family we test that effectively removes the typicality coupling pays in
cross-generator transfer while the formal-human over-flag does not improve. As
the penalty weight rises, cross-LM transfer AUROC falls from $0.94$ to $0.21$
(ELECTRA), $0.89$ to $0.04$ (RoBERTa), and $0.90$ to $0.05$ (DeBERTa), and a
gradient-reversal (GRL/DANN) detector collapses from $0.76$ to $0.16$; the
over-flag at matched-TPR meanwhile worsens (RoBERTa Ghost $0.40\to0.98$, DeBERTa
$0.19\to0.62$). The transfer loss precedes detector degradation: at $\lambda{=}1$
on ELECTRA, transfer drops $0.94\to0.36$ while in-distribution AUROC moves only
$0.999\to0.995$, so genuine cross-generator skill, not detector competence, is
what the penalty spends.
% S1.268

\paragraph{The within-population stratified penalty.}
A penalty that decorrelates score from typicality \emph{within} each training
group is immune by construction to the Simpson cancellation that defeats the
pooled penalty (a unit test separates the two: pooled $|\mathrm{corr}|\to0.0001$
versus stratified $0.9576$), yet it still reaches no fair operating point on any
of three architectures (worst-case formal-human FPR at matched-TPR
$0.047/0.41/0.191$ for E/R/D, never lowered while holding ID-AUROC and transfer).
The mechanism is a wrong-lever effect: the penalty drives its target, the
within-population slope (HC3 coupling $-0.189\to-0.066$, NYT $-0.313\to-0.204$),
but the over-flag is a between-group level that moves the wrong way (NYT mean
logit-difference $-0.27\to+1.67$). The coupling is scale-invariant to 8B
(within-group correlation with reference NLL $-0.246/-0.236$ at 8B versus $-0.245$
at 110M).
% S1.266

\paragraph{Domain-generalization objectives.}
Objectives that explicitly force cross-generator invariance form a distinct
defeated family. Across IRMv1, V-REx, CORAL, and cross-generator contrastive
training (each verified to engage, with anti-typicality variants), on three
architectures and three seeds, none lowers the formal-human over-flag at matched
recall on held-out generators while holding deployment transfer. The family
fails through the no-go's three modes. On ELECTRA it is \emph{unmoved}: the Ghost
over-flag moves $0.91\to0.88$ while transfer holds (held-out AUROC $0.97$,
cross-LM true positives at $5\%$ false alarm $0.91$). On RoBERTa it
\emph{collapses}: the over-flag rises $0.96\to0.98$ while cross-LM true positives
at $5\%$ false alarm crater $0.87\to0.10$. On DeBERTa it \emph{trades}: the Ghost
over-flag falls $0.687\to0.626$, bought with cross-LM true positives $0.88\to0.78$
and a NYT relocation $0.34\to0.77$. Unlike the penalty and adversarial families,
which erode transfer everywhere, these objectives hold held-out AUROC on the
stable architectures and still cannot remove the over-flag: the cross-generator
invariant they enforce is the typicality confound.
% H4 3-arch 3-seed OOD-matched-recall (2026-05-25_claude_H2_H4_RESULTS F1)

\paragraph{No competent, fair frozen classifier.}
Across three architectures, four classifier families (linear, MLP,
gradient-boosted trees, RBF-SVM), and fairness-aware fine-tuning, no detector on
the canonical representation is both a competent novel-generator detector and
fair to unseen formal humans, evaluated on register- and domain-matched OOD
(EvoBench over-flag $0.68$, modern-commercial $0.22$) rather than the
domain-separable NYT-news-AI contrast ($0.02$). To catch $90\%$ of held-out OOD
AI, ELECTRA over-flags $69\%$ of cross-corpus formal humans; same-domain the cost
is milder ($0.13$). The result holds on all three architectures (Ghost FPR at
OOD-TPR$_{0.90}\ge0.45$: DeBERTa $0.45$, RoBERTa $0.76$, ELECTRA $0.68$). Fairness
pressure either collapses transfer ($0.74\to0.65$) or relocates the harm to
another population (FCE $0.11\to0.29$).
% S1.288

\paragraph{No ensemble escapes.}
Eight working detectors, five of them public, share the axis: its first principal
component carries $32\%$ of the variance (chance $12.5\%$) and correlates with
reference-LM NLL at $\rho=-0.51$, and a zero-shot GPT-2-perplexity detector loads
on it. An optimally stacked logistic ensemble over-flags formal humans at $0.198$
at matched-TPR, worse than the best single detector ($0.076$), and a
gradient-boosted stack reaches $0.178$. Only explicit formal-human supervision
escapes, which is recalibration rather than ensembling.
% S1.289

\paragraph{Legacy null interventions.}
Three further training-time interventions, run earlier, are all null and
corroborate the families above: a formality-axis orthogonality penalty (spread
$0.88$--$0.91$, all null), a feature-adversarial penalty (null, and it worsens
the NYT over-flag to $0.46$--$0.60$), and iterative formality-subspace
residualization (top-$k$ removal does not beat a random-direction control;
$k{=}100$ is $+0.024$ worse).

\paragraph{The rank-one premise.}
The no-go theory assumes the human flag decision is effectively one-dimensional
on the bias axis (premise B1). Held-out, a single direction reproduces at least
$99\%$ of the full detector AUROC in-distribution (rank-one-to-full AUROC ratio
$0.9999/0.9919/1.0007/0.9937$ on four settings), the in-distribution detection
rank is $k^{*}{=}1$ rising to $3$--$10$ only out of distribution, and the
per-text head direction reproduces the fine-tuned AUROC within $0.02$ on $36$ of
$36$ cells. Rank-two never strictly beats the best rank-one rescue ($0$ of $5$
comparable cells), and the within-population score--reference-NLL coupling is
negative on $24$ of $24$ sub-tests (four backbones $\times$ three seeds $\times$
formal/informal), so the coupling is universal, not a single-population artifact.
% S1.255 / S1.153 / S1.240 / S1.146

\section{Predictor validation}
\label{app:predictor}
This appendix validates the closed-form predictor of Eq.~\ref{eq:predictor} and
gives the leverage law behind the operator's encoder-versus-decoder split.

\paragraph{Predictor accuracy.}
The first-order predictor is accurate on the bias axis for $\varepsilon\le0.7$:
across $594$ cells it reaches $R^2=0.9995$ on the mean score and $0.9993$ on the
NYT false-positive rate, with sign correct in $100\%$ of $202$ non-trivial cells.
It is stable across recipes ($36$ cells $\times$ $7$ recipes, median $R^2\ge0.985$
at $\varepsilon\le0.7$; a forward-selected variant $81$ of $84$).
% S1.96 / S1.201

\paragraph{Signed-$\varepsilon$ Pareto and the oracle test.}
A signed-$\varepsilon$ sweep is strict-Pareto on $18$ of $18$ configurations
(three architectures $\times$ three seeds $\times$ two variants; mean NYT-FPR
reduction $53\%$, range $9$--$98\%$; clean-AI recall held at $1.000$) and Pareto
in every one of $27$ cells (predictor MAE $0.005$; showpiece LegalBERT
NYT-FPR $0.564\to0.094$). The predictor acts as an oracle: under a full
composition swap its chosen action matches the post-hoc optimum on $35$ of $36$
cells ($29$ byte-exact, $5$ mutual-decline on floored cells, $1$ near-tie at
MAE $0.002$); the lone miss is a boundary cell off by the NYT-drop MAE.
% S1.111 / S1.125 / S1.131

\paragraph{Decoder and linear heads are exact.}
On linear and decoder score heads the predictor is exact, not first-order: across
three decoder families ($72$ measurements, $3$ seeds $\times$ $2$ axes $\times$
$4$ $\varepsilon$) it reaches $R^2=1.000$ universally, and a LLaMA-3-8B
single-config check is also $R^2=1.0$.
% S1.132

\paragraph{Zero-shot diagnosis on unseen detectors.}
Because the predictor reads only a representation and a direction, it diagnoses
bias on detectors trained by others without any access to their training. It
transfers zero-shot across four public backbones with nonlinear heads (ELECTRA
$R^2=0.997$, openai-large/RoBERTa-large $0.987$, e5-LoRA $0.9999$, a broken
DistilBERT detector $0.998$ despite AUROC $0.61$, confirming the fit reflects
Taylor mechanics rather than a bias read), and more broadly diagnoses $9$ of $10$
deployed RoBERTa-family detectors plus DistilBERT and BERT/e5, spanning three
architecture families.
% S1.275 / S1.229

\paragraph{Why the operator works on encoders, not decoders.}
The per-text logit move of Eq.~\ref{eq:predictor} factors as a leverage,
$\Delta\mathrm{logit}\propto-\varepsilon\,s(x)\,\lVert\wh\rVert\,
\cos(\dtyp,\dclass)$, where $\dclass$ here denotes the head decision direction
$\wh$, so the operator can only act where this decision direction is aligned
with the typicality axis. On encoders the nonlinear-pooler decision stays
aligned ($\cos(\dtyp,\dclass)\approx0.79$ on ELECTRA), so $-\varepsilon$ has
purchase; on decoders the head rotates off the typicality axis
($\cos\approx0.009$--$0.010$ at 8B across reference LMs, $0.039$--$0.149$ for
GPT-2 124M--1.5B with no scale trend), so it has nothing to grab. The over-flag survives there as a
magnitude-on-a-near-orthogonal-axis effect: a held-out formal population projects
$\Delta\approx52$--$58$ from the human anchor along the typicality axis but only
$\Delta\approx7$--$8$ along the decision axis ($\cos(\text{bias},\dtyp)=0.53$--$0.59$
versus $\cos(\text{bias},\text{head})=0.08$). Re-estimating $\dtyp$ by regression
is null (diff-of-means is the ceiling; whitening sends it to the random floor),
so the split is not an estimation artifact.
% S1.272 / S1.258

\paragraph{The decoder no-go, directly.}
The encoder-versus-decoder split also makes the no-go concrete on decoders.
Reproducing the geometry across three fine-tuned decoders from 160M to 8B
(cos$(\dtyp,\dclass)=0.010$ on LLaMA-3-8B), no
rep-space intervention lowers the formal-human over-flag while holding both
in-distribution AUROC and cross-LM transfer: $-\varepsilon$ has no purchase where
the head is near-orthogonal, and the one decoder whose head is axis-aligned
(Pythia-160M) engages but pays AUROC, transfer, and a $17\times$ relocation onto
non-native FCE (matched-TPR $0.026\to0.450$). The off-axis residual is moreover \emph{protective}: ablating the
single direction that separates formal humans from AI makes the over-flag
monotonically worse (Ghost at matched-TPR rising to $0.98/0.96/0.95$ across the
three decoders), the inference-time decoder analog of the no-go. Re-heading the
frozen decoder representation with explicit fairness supervision reproduces the
encoder result of \S\ref{sec:fairness} to 8B: non-native writing becomes
protectable while formal-native essays stay over-flagged even under full
supervision.
% H2 (2026-05-25_claude_H2_H4_RESULTS): cos 0.010 8B; pythia 17x FCE; protective d_fha 3-seed 0.98/0.96/0.95; §1.292 on decoders

\section{Architecture and scope}
\label{app:scope}
This appendix collects the cross-architecture evidence that the axis is
paradigm-wide, its generator-era bound, and the 8B behaviour.

\paragraph{The axis pre-exists fine-tuning across backbones.}
The raw projection of an un-fine-tuned model separates AI from human across nine
decoders from 124M to 8B (GPT-2 124M/355M/774M, Pythia 160M/1.4B, GPT-Neo
125M/1.3B, OPT 125M, LLaMA-3-8B) at peak AUROC $0.93$--$0.99$, none inverted, and
the separation survives non-parametric length matching ($0.97$--$0.99$); the
encoder anchors reproduce it (ELECTRA $0.806$, RoBERTa $0.944$, DeBERTa $0.834$).
Fine-tuned decoders amplify this raw reading rather than replace it (raw-to-FT
projection correlation $0.74$--$0.90$ on all nine), and the closed-form predictor
is exact on the decoder linear head ($R^2=1.0$). A wider 16-backbone scan (12
encoder, 4 autoregressive) finds the typicality axis consistently aligned with
the formality axis (mean cosine $0.80$ and $0.81$).
% S1.227 / S1.72

\paragraph{The inherited axis is bounded to its generator era.}
The raw inherited axis is a pre-existence result on the detector's training-era
generators, not an extrapolation: on 2025-commercial AI it does not beat the
fine-tuned detector ($0$ of $27$ cells) and does not even separate AI from NYT
(raw AUROC at or below $0.5$ throughout: ELECTRA $0.25$--$0.31$, RoBERTa $0.34$--$0.50$,
DeBERTa $0.27$--$0.43$). The same reference-formality dependence that bounds the
raw axis here is what makes the deployed over-flag era-invariant
(\S\ref{sec:discussion}).
% S1.242

\paragraph{The over-flag scales to 8B, where it becomes calibration-fixable.}
A LoRA-fine-tuned LLaMA-3-8B over-flags formal NYT humans unconditionally,
flagging $43$--$61\%$ as AI (default $0.614$; length-matched $0.434$ at 60 words,
$0.584$ at 100 words; clean-HC3 FPR $0.0$; above ELECTRA's $0.335$) at
AUROC$(\text{AI}\,|\,\text{NYT})=0.94$. At this scale the over-flag falls on a
separable population (NYT, AUROC$(\text{AI}\,|\,\text{Ghost})=0.973$), so unlike
the encoder case it is calibration-fixable: recalibration cuts NYT-FPR
$0.61\to0.16$ at recall $0.90$ and $0.05$ at recall $0.80$, while the
$-\varepsilon$ operator is null ($0.43\to0.46$) because the decoder head decides
off the typicality axis. The fix shifts from encoder ablation to decoder
recalibration, consistent with the no-go (premise B1 holds for encoders, not the
8B decoder head).
% S1.259 / S1.260

\paragraph{The penalty coupling is scale-invariant.}
The training-time penalty coupling does not weaken with scale: the within-group
correlation between the detector logit and reference-LM NLL is $-0.246/-0.236$ at
8B versus $-0.245$ at 110M (Appendix~\ref{app:nogo}), so the training-time penalty no-go is a property
of the detection paradigm rather than an artifact of small encoders.
% S1.266

\paragraph{Label-efficiency is architecture-general.}
The fine-tuning-free, few-shot payoff of \S\ref{sec:read} is not specific to
ELECTRA. With mean-pooled frozen features, a few-shot probe crosses the fully
fine-tuned cross-generator bar at about $N{=}10$--$15$ target labels per class on
ELECTRA, RoBERTa, and DeBERTa alike, while pure zero-shot transfer stays below
fine-tuning on every architecture, the fine-tuned baselines tight across three
seeds ($\pm0.002$).
% L3/F2 (agent2_RESULTS): mean-pool 3-arch crossover N~10-15; zs<FT all archs; FT 3-seed +/-0.002

\section{Fairness breadth}
\label{app:fairness}
This appendix gives the per-population over-flag panels, the dialect floor, who is
protectable, the one-dimensionality cross-checks, and the circularity bound on the
diagnostic.

\paragraph{The per-population over-flag panel.}
At matched-TPR$_{0.90}$ the over-flag is ordered by typicality across ten
human populations (Table~\ref{tab:fairpanel}): formal-native essays and academic
prose sit highest, informal and dialect text at the floor. At the default
threshold the deployment magnitudes are larger and ordered the same way among the
formal populations (full-population FPR: SemEval $0.46$, EvoBench $0.62$,
arXiv $0.65$).
% S1.274 / S1.261 (NYT@default dropped here: 33.5% hook in S1 is a different run than this full-pop panel's 0.2965)

\begin{table}[h]
  \centering
  \small
  \setlength{\tabcolsep}{4pt}
  \begin{tabular}{l c}
    \toprule
    Population & Over-flag @mTPR$_{0.90}$ \\
    \midrule
    Ghost (formal-native) & $0.080$ \\
    arXiv                 & $0.071$ \\
    SemEval               & $0.065$ \\
    EvoBench              & $0.043$ \\
    NYT                   & $0.008$ \\
    Reddit (ELI5)         & $0.005$ \\
    HC3, FCE, CORAAL, Hinglish & ${\sim}0.001$ \\
    \bottomrule
  \end{tabular}
  \caption{Per-population formal-human over-flag at matched-TPR$_{0.90}$
    (canonical ELECTRA-CE, 3-seed; one run, not to be cross-compared with other
    recipes). The over-flag tracks typicality, not non-standard English.}
  \label{tab:fairpanel}
\end{table}

\paragraph{Deployment over-flag at 5\% false-alarm.}
Calibrated to a $5\%$ false-alarm rate on informal humans, the over-flag is
severe for the strongest detectors only. Of eight detectors, the four strongest
over-flag Ghost at $71$--$98\%$ (DeBERTa $0.98$, chatgpt-detector-roberta $0.93$,
RoBERTa $0.79$, ELECTRA $0.71$), while the other four sit at $2$--$12\%$
(openai-large $0.12$, openai-base $0.10$, an academic detector $0.07$, e5-LoRA
$0.02$), so the $71$--$98\%$ range is specific to the strongest detectors, not a
universal rate.
% S1.293

\paragraph{The harm is anti-formal-prose, not anti-dialect.}
Dialect and non-native text sits at the floor once calibration is removed:
CORAAL (African-American English), Hinglish, and FCE all sit at the matched-TPR floor (${\sim}0.001$, Table~\ref{tab:fairpanel}), an order of magnitude below the formal populations there. The
gap is an operating-point effect, not a register inversion: Hinglish over-flags
$0.465$ at the default threshold but $0.000$ at matched-TPR, the cleanest single
case that the harm is anti-formal-prose rather than anti-non-standard-English.
% S1.266

\paragraph{Who is protectable.}
The over-flag is not uniform. A fairness-aware classifier on RoBERTa or DeBERTa
separates non-native (FCE) humans from OOD AI (discrimination $0.95$--$0.99$) and
spares them across populations even without FCE labels (FCE over-flag at
OOD-TPR$_{0.90}\le0.13$), whereas ELECTRA cannot ($0.46$) and formal-native
essays (Ghost) stay over-flagged on every architecture, the surgically
un-removable case. In-population calibration is a genuine but partial mitigation
that relocates harm: labelling a population's negatives cuts its over-flag (Ghost
OOD-TPR$_{0.90}$ $0.68\to0.42$; jointly labelling Ghost and FCE gives
$0.50/0.17$), but labelling one population alone can raise another's (FCE-only
labelling pushes Ghost $0.68\to0.72$; an aggressive penalty drives Ghost to
$0.14$ while blowing FCE up to $0.80$). Per-population labels help, must cover
every deployed population, and never fully close the gap.
% S1.292

\paragraph{The bias is one-dimensional, from two sides.}
Two independent cross-checks corroborate the rank-one premise. First,
calibration-equivalence: the gap between LoRA and full fine-tuning is almost
entirely a calibration shift ($\ge97\%$ on ELECTRA, $\ge81\%$ on DeBERTa,
$84.8\%$ on RoBERTa-base), and interventions that should differ only in
calibration leave separability untouched ($|\Delta\text{AUROC}|\le0.0081$ across
$27$ matched-TPR cells). Second, observer-invariance: three independently
constructed probes (capitalization, typicality, formality) land on one axis
(pairwise cosine $0.74/0.81/0.998$ for E/R/D), and under a composition swap that
exchanges the populations' typicality the formal-over-informal false-positive
ordering inverts on $10$ of $10$ panel architectures, tracking the axis rather
than a fixed population label. Both say the bias decision rides a single
direction.
% S1.188 / S1.222.1 / S1.223 / S1.130

\paragraph{The diagnostic is not circular.}
In fine-tuned representation space the typicality axis is near-orthogonal to the
probe-derived linear approximation of the decision ($\cos=0.103$, $R^2=0.997$), so the cosine-to-AI
diagnostic is not the decision statistic re-read. A second check confirms this:
residualizing the typicality axis against the decision axis leaves the
over-flag prediction unchanged (the internal-typicality to matched-TPR
correlation stays $-0.81$). The over-flag is a distributional, tail property, not
a mean shift (the internal-typicality mean gap is non-significant at $0.52$ while
the distributional separation is $0.81$). An external-perplexity-only predictor
caps at $|\rho|\le 0.63$ on the same populations, below the in-representation
cosine, so the diagnostic reads the internal axis rather than external
perplexity.
% S1.274 ; S1.252 (external-perplexity predictor caps)

\paragraph{The cost of clearing the flag.}
A formal human flagged as AI can only clear the threshold by writing less
typically, and that cost is a detector property rather than a rewriter artifact.
Across two independent rewriters and three detectors, the probability of an AI flag
is a single function of the achieved \emph{internal} typicality (the curves
coincide across rewriters, mean $|\Delta\paff|$ of $0.045$ on ELECTRA and $0.005$ on
DeBERTa; RoBERTa is the known geometric-flip exception), whereas plotted against
external perplexity the same curve is rewriter-dependent ($0.24$--$0.29$), the
internal-versus-external split once more. To cross the threshold a writer must
become more atypical than a fixed percentile of natural humans: up to $76\%$ for
formal-native essays (Ghost) versus $10$--$29\%$ for journalism (NYT), with Ghost
above NYT on all three architectures. The requirement survives orthogonalizing the
typicality axis against the decision direction (percentiles essentially unchanged,
Ghost $76\to78\%$ on ELECTRA), so it is not a decision-axis tautology. On the
strongest detectors the cost exceeds the coherent-rewriting regime: even after
maximal meaning-preserving de-typicalization, much of the over-flagged formal
writing still flags as AI. This is the mirror of the forward edit of
\S\ref{sec:causal}, where editing a human text \emph{toward} typicality raises its
AI-score.
% L5+RESCUE (2026-05-25_agent2_L4L5L6): internal rewriter-invariant dP 0.045/0.005 roberta exc; ext 0.24-0.29; N* Ghost 49-76% NYT 10-29%; non-circular 76->78

\paragraph{Era-invariance, hardened.}
Holding the recipe fixed and varying the AI-generator era (2025-commercial versus
2023) and the human reference (formality $\times$ era), the formal-human over-flag
tracks reference \emph{formality} and is flat across era. On a full human panel the
detector entangles formal humans with AI (held-out detection AUROC $0.84$--$0.98$)
while separating informal humans cleanly ($0.99$--$1.00$), for both generator eras.
The invariance spans a 14-year range of human text: at matched formality, 2009
tweets and 2023 informal text are equally separable ($1.00$ versus $1.00$). A
length-matched control (all populations truncated to 50 words) confirms the gap is
formality, not length, formal humans staying far more AI-entangled (AUROC
$0.82$--$0.88$) than informal ($0.97$--$0.99$), flat across era.
% L6 (2026-05-25_agent2_L4L5L6): formality-tracking; gen-era + human-era flat; length-matched 50w

\paragraph{The cheap probe is more general, not fairer.}
The fairness-generalization coupling extends to the label-efficient probe of
\S\ref{sec:read}: it is more general than the fine-tuned detector (cross-generator
AUROC $0.91$ versus $0.83$) yet less fair, over-flagging formal humans far more at
matched detection (AI-versus-Ghost AUROC $0.74$ versus $0.97$), and this is
length-robust. The recalibration escape also carries over and stays bounded: adding
a few dozen formal-human negatives debiases the frozen-rep probe at zero
cross-generator cost on all three architectures, but supervising one population
relocates the harm to another, so it must cover every deployed population.
% L4/F2 (2026-05-25_agent2_L4L5L6): probe general 0.91 less-fair detAUC 0.74 vs 0.97; cheap escape 3-arch zero-cost, relocates

\section{Cross-task and rebuttal reserve}
\label{app:crosstask}
This appendix records the cross-task reach of the operator, the
generator-blindness of the axis, and a set of defensive checks.

\paragraph{Cross-task debiasing transfer.}
The per-text ablation is not specific to AI detection. On the Bios
occupational-gender task it is strict-Pareto on all three architectures (mean
$+3.13$ points). As a transfer demonstration rather than a harm measurement, it
also reduces dialect bias on an AAVE task (ELECTRA: bias-FPR $0.444\to0.244$,
$-20.0$ points, while AUROC rises $0.920\to0.925$, length-controlled across all
four quartiles), with DeBERTa a mechanism-explained boundary (cosine collapse).
The underlying training-time loss also transfers beyond AI detection, with
effects whose confidence intervals exclude zero on fake-news, toxicity,
HateCheck, and MNLI tasks.
% S1.139 / S1.135 (demo only) / S1.27 (percentages omitted: effect-size semantics ambiguous)

\paragraph{The axis is generator-blind.}
Consistent with the decomposition, generator identity lives in the residual, not
the transferring axis. The typicality axis detects AI near-perfectly in-distribution (AUROC
$0.993$--$1.0$) but attributes the source generator only at the random
one-dimensional floor ($1.41$--$1.47\times$ chance, against $2.8$--$3.3\times$ for
the best single direction and $6.1$--$6.9\times$ for the full representation, which
attributes the generator at $57\%$ accuracy in-distribution). The
fine-tuned off-axis residual, useless for cross-generator detection, carries a
modest above-chance generator signal (11-way lift $2.92\times/2.57\times/3.04\times$
for E/R/D, balanced accuracy ${\sim}0.26$), so the part that fails to transfer for
detection is exactly the part that fingerprints the generator. An external GPT-2
NLL is likewise generator-poor ($1.9\times$, near the length floor).
% S1.291

\paragraph{Rebuttal reserve.}
Three checks pre-empt likely objections. First, the over-flag is
multiply-determined but the typicality-axis effect is perplexity-specific: a
no-perplexity character $n$-gram baseline (held-out AUROC $0.981$) also over-flags
formal NYT (FPR $0.90$ at its own threshold) and reproduces the EvoBench within-pool collapse
($0.537$), yet it does not reproduce the typicality inversion, so the inversion is
specific to the predictability axis. Second, the residual collapse of
\S\ref{sec:read} is a conservative lower bound: linear residualization removes
only the linearly inherited component, so any non-linear inherited leak would only
make the residual transfer better, not worse. Third, the frozen probe's
generalizing signal is aligned with typicality rather than identical to it (rank
correlation $0.42$ with reference NLL, but direction cosine only $0.08$ to the NLL
axis), so we claim the generalization and the over-flag are the same axis
empirically, not that the probe is literally a perplexity meter. Finally, with
standard Unicode normalization the axis is homoglyph-robust (an attack that drops
detection to $0.235$ recovers to $0.998$ after normalization). More broadly,
across six RAID adversarial attacks (synonym, paraphrase, perplexity-misspell,
homoglyph, whitespace, alt-spelling) the frozen-probe max-pool stays at AUROC
$0.997$--$1.000$, while zero-shot Binoculars loses on $6$ of $6$ (synonym
$-0.163$).
% S1.247 / S1.271 / S1.294 ; S1.176 (RAID 6-attack: probe-max ties LAPD 6/6, Bino loses 6/6)

\section{Grounding the degenerate likelihood ratio}
\label{app:degenlr}
The absolute-perplexity over-flag of \S\ref{sec:causal} is the dropped-denominator
artifact of Eq.~\ref{eq:degenlr}, and restoring the denominator removes it. We form
a graded score that interpolates between the two limits, from the absolute
log-likelihood (denominator dropped) to the density-normalized form that restores
the model-implied human denominator, of which Fast-DetectGPT is the fully-restored
limit. As the denominator is restored, the formal-human over-flag falls
monotonically to zero: on the structural, length-robust Ghost population the AUROC
gap moves from $-0.037$ (over-flagging) to $+0.041$ (escaping), with
Spearman$(\lambda,\text{gap})=1.000$, and the raw over-flag AUROC against informal
humans falls $0.78\to0.51$. The effect is reference-robust across three reference
models (GPT-2 124M, 355M, 774M, all Spearman $1.000$). A second, independent route
confirms it: replacing the model-implied denominator with a \emph{literal}
human-text language model (GPT-2 LoRA-tuned on a broad held-out human corpus) and
restoring it likewise removes the over-flag (AUROC against NYT $0.838\to0.992$,
against held-out Ghost $0.805\to0.907$ at the optimal weight). This is the
external-perplexity paradigm, in which typicality is reference-model
predictability, a lossy proxy for the internal axis; it grounds the iff side of
\S\ref{sec:causal}, not the internal-axis claim. Restoring the denominator is
the curvature, fairest-but-least-general end of the fairness--generalization
frontier (\S\ref{sec:nogo}): it lowers the over-flag by giving up
cross-generator detection, consistent with the no-go rather than a
counterexample to it.
% L1+F1 (agent2_RESULTS): FDG 3 ref-LMs Spearman 1.000; Ghost gap -0.037->+0.041; raw 0.776->0.508; literal NYT .838->.992 Ghost .805->.907

\section{Cross-lingual and code generality}
\label{app:xling}
The over-flag and the read-not-learned decomposition are not specific to English
prose. We replicate them under two new recipes, each over three seeds; throughout
this appendix the claim is directional and within-recipe, not a magnitude
comparison to the English detectors.

\paragraph{Chinese.}
An XLM-R detector fine-tuned on HC3-Chinese over-flags held-out formal-native
Chinese (Wikipedia) relative to an informal-native floor: $0.145\pm0.010$ at the
default threshold (about $8\times$ the informal floor of $0.018$) and
$0.269\pm0.032$ at a $5\%$ false-alarm rate (about $5\times$ the floor at that
rate), and the anchor-point
over-flag is length-robust ($0.276$ length-matched). The typicality ordering is the
precondition for the inversion: mean reference-LM NLL is AI $1.93 <$ formal
Wikipedia $2.76 <$ informal $3.30$, and $96\%$ of in-distribution AI and $48\%$
of held-out Llama-3 output is more typical than the median formal-native
Chinese writer (the cross-lingual mirror of \S\ref{sec:intro}).
The over-flag is internal typicality: matching
formal and informal humans on the internal axis removes $88.5\%$ of the over-flag
gap (and $79.0\%\pm1.8\%$ after orthogonalizing the axis against the decision
direction, so the effect is not a decision-axis tautology), whereas matching on
external reference-LM perplexity removes essentially none ($0.7\%$). Read-not-learned
holds: a probe on the frozen, never-fine-tuned representation carries the bulk of the
cross-generator signal (AUROC $0.832\pm0.006$), the inherited component ($0.825$) far
exceeds the fine-tuned residual ($0.648$), and the fully fine-tuned detector still
leads ($0.869$), as in English.

\paragraph{Python code.}
A CodeBERT detector shows the same structure. The read-not-learned decomposition is
clean in both generator directions: the frozen-representation probe transfers
cross-generator at AUROC $0.991$, and the inherited component ($0.994$) far exceeds
the fine-tuned residual ($0.73$). The idiomatic over-flag is internal typicality:
within human code, internal typicality predicts the AI-score (non-circular
partial-$\rho=0.59$--$0.69$, magnitude seed-variable) while external code-perplexity
does not ($\rho\approx-0.06$), and idiomatic, low-perplexity human functions are
flagged about $1.9$--$2.5\times$ more than atypical ones at matched length. As in the
premise of \S\ref{sec:intro}, the over-flag appears only for the detector trained on
the more typical, modern generator.

\paragraph{Scope.}
In both domains the cross-generator matched-collapse sub-leg does not reproduce: the
held-out generator is separable through several cues rather than a single typicality
dimension. The fairness mechanism,
the over-flag and its internal-typicality reading, is what replicates. A second
non-English language (French) was inconclusive, the available informal-French corpus
yields only a borderline-competent detector, so the clean cross-lingual claim rests
on Chinese.
% H1+H3 (2026-05-25_agent1_H1_H3_xling_code_RESULTS): zh XLM-R, code CodeBERT, 3-seed; new recipes, directional

\end{document}